\documentclass{article}

\usepackage{microtype}
\usepackage{graphicx}
\usepackage{subcaption}
\usepackage{paralist}
\usepackage{booktabs} 

\usepackage{url}
\usepackage{hyperref}
\usepackage{xurl}

\usepackage[preprint]{icml2026} 

\usepackage{amsmath}
\usepackage{amssymb}
\usepackage{mathtools}
\usepackage{amsthm}

\usepackage[capitalize,noabbrev]{cleveref}

\theoremstyle{plain}
\newtheorem{theorem}{Theorem}[section]

\newtheorem{lemma}[theorem]{Lemma}

\theoremstyle{definition}
\newtheorem{definition}[theorem]{Definition}

\theoremstyle{remark}

\usepackage[textsize=tiny]{todonotes}
\usepackage{amsmath,amsfonts,amsthm,bm}

\def\ceil#1{\lceil #1 \rceil}
\def\floor#1{\lfloor #1 \rfloor}
\def\1{\bm{1}}

\def\vh{{\bm{h}}}

\def\vx{{\bm{x}}}
\def\vy{{\bm{y}}}

\def\mI{{\bm{I}}}

\def\mW{{\bm{W}}}

\DeclareMathAlphabet{\mathsfit}{\encodingdefault}{\sfdefault}{m}{sl}
\SetMathAlphabet{\mathsfit}{bold}{\encodingdefault}{\sfdefault}{bx}{n}

\def\sD{{\mathbb{D}}}
\def\sR{{\mathbb{R}}}

\icmltitlerunning{Do Neural Networks Lose Plasticity in a Gradually Changing World?}

\begin{document}

\twocolumn[
\icmltitle{Do Neural Networks Lose Plasticity in a Gradually Changing World?}

\icmlsetsymbol{equal}{*}

\begin{icmlauthorlist}
\icmlauthor{Tianhui Liu}{ua}
\icmlauthor{Lili Mou}{ua,cifar}
\end{icmlauthorlist}

\icmlaffiliation{ua}{Dept. Computing Science \& Alberta Machine Intelligence Institute (Amii), University of Alberta}
\icmlaffiliation{cifar}{Canada CIFAR AI Chair}

\icmlcorrespondingauthor{Tianhui Liu}{tianhuiliu17@gmail.com}
\icmlcorrespondingauthor{Lili Mou}{doublepower.mou@gmail.com}

\icmlkeywords{Machine Learning, Continual Learning, Loss of Plasticity, arXiv Preprint}

\vskip 0.3in
]

% Print affiliations in the footnote
\printAffiliationsAndNotice{} 

\begin{abstract}
Continual learning has become a trending topic in machine learning. Recent studies have discovered an interesting phenomenon called loss of plasticity, referring to neural networks gradually losing the ability to learn new tasks. However, existing plasticity research largely relies on benchmarks with abrupt task transitions, without examining whether the abruptness itself contributes to the observed plasticity loss. In this paper, we investigate the role of transition abruptness by simulating gradually changing environments through input/output interpolation and task sampling. We perform theoretical and empirical analysis, showing that the severity of plasticity loss is closely tied to the abruptness of task transitions, and can be substantially reduced when the environment changes gradually. 
\end{abstract}

\section{Introduction}
\label{intro}

Continual learning has become a trending topic in machine learning research, where a system incrementally learns from non-stationary distributions~\citep{Ring1997,ella_lifelong,deep_gen_replay,pnasforget,abbaslp_RL}. 
In supervised learning, a system is trained to perform isolated tasks~\citep{Thrun1998}. However, training a system from scratch is onerous and time-consuming, especially for large neural networks~\citep{scaling}. It would be more computationally and economically efficient to incrementally adapt them to new tasks.

Loss of plasticity is an intriguing phenomenon recently discovered in continual learning, referring to neural networks gradually losing the ability to learn new tasks~\citep{dohare2021continual, Dohare2024, lyle_understand, mutual_frozen, abbaslp_RL, lewandowski2025learning}. Proposed mitigation techniques include regularizing the model weights~\citep{mutual_frozen,lewandowski2024directionscurvatureexplanationloss,lewandowski2025learning,kumar2024maintaining}, using alternative activation functions~\citep{ berariu_study_plasticity, lee2023plastic, abbaslp_RL}, controlling the loss landscape sharpness \citep{lyle_understand}, resetting less-used neurons \citep{Dohare2024, dormantSokar}, and selectively forgetting memorized noise \citep{shin2024dash}.

\textcolor{black}{A common thread across these studies is the use of benchmarks with
abrupt task transitions; e.g., random label assignment, random pixel
permutation~\citep{Dohare2024,lyle_understand,lewandowski2024directionscurvatureexplanationloss,lewandowski2025learning}.
These benchmarks have shaped our understanding of plasticity loss,
but they mostly assume abrupt task transitions. Such a design choice
has not been carefully examined. This leaves open an important question
of how the abruptness of task transitions affects plasticity.}

It is therefore important to consider whether such contrived setting of abrupt task
transitions reflects real-world continual learning. Consider the evolution of human language. Researchers show that the sense of a word may change from time to time, but this happens gradually, as the new and old senses typically co-exist during the change~\citep{kilgarriff1997don,word_sense,covid_sense}. 

\textcolor{black}{This distinction matters because real-world non-stationarity is rarely abrupt.} Consider the evolution of human language. Researchers show that the sense of a word may change from time to time, but this happens gradually, as the new and old senses typically co-exist during the change~\citep{kilgarriff1997don,word_sense,covid_sense}. For instance, the word ``sick'' traditionally refers to ``illness'' in standard English; however, it has recently adopted new meanings such as ``cool'' or ``crazy''~\citep{word_sense}. This shift marks a transition from a negative sense to a positive one, with both meanings currently being understood by the general public.

\textcolor{black}{In this paper, we investigate the hypothesis that gradually changing environment preserves plasticity. We propose to perform input/output interpolation or task sampling to simulate a gradually changing environment. We vary the degree of transition abruptness while keeping all other aspects of the benchmark fixed. Our empirical results across four benchmarks covering both trainability and generalizability (which are two aspects of plasticity) reveal that plasticity loss is highly sensitive to transition abruptness and that neural networks preserve plasticity magnitudes longer than those under abrupt task changes. In addition, we provide theoretical analysis based on the intuition that abrupt task changes cause abrupt shifts in the loss landscape, trapping parameters in poor local optima. In contrast, gradual transitions allow the landscape to evolve smoothly, guiding parameters toward better solutions for each new task.}

\textcolor{black}{Our findings have two practical implications. First, they suggest that future plasticity benchmarks should explicitly control for transition abruptness, as current benchmarks may overemphasize the severity of plasticity loss relative to realistic conditions. Second, for scenarios where task transitions are genuinely abrupt (e.g., robots transferred to new environments), our interpolation and task sampling protocols offer simple yet effective techniques that effectively mitigate plasticity loss.}

\section{Related Work}

Plasticity represents a neural system's adaptability to new environments. This mechanism has been a foundational subject in neuroscience for decades~\cite{citri2008synaptic,rodriguez2009astroglia}, but has only recently received growing attention in supervised classification and reinforcement learning~\cite{dohare2021continual,Dohare2024, lyle2022understanding, dormantSokar}.

The loss of plasticity is an intriguing phenomenon observed in deep learning models, which can be categorized into two aspects: the neural network's ability to train \cite{dohare2021continual,Dohare2024,lyle2022understanding,elsayed2024addressing, lewandowski2024directionscurvatureexplanationloss, lewandowski2025fourier,lewandowski2025learning,ma2024revisiting}, and the ability to generalize \cite{Dohare2024,warm_adams, mutual_frozen, lewandowski2025learning, slow_steady,anonymous2026activationdesign}. Early research primarily focuses on trainability, operating under the assumption that minimizing training loss is the principal objective. However, \citeauthor{warm_adams} find that neural networks initialized with pre-trained weights leads to faster convergence but worse generalization \cite{warm_adams}. This indicates that the network has become trapped in a low-generalization basin, memorizing without learning robust representations.

Prior studies have hypothesized different underlying causes of plasticity loss and proposed corresponding methods to preserve it. Loss of plasticity is initially attributed to dormant neurons~\cite{dormantSokar, Dohare2024}, where the number of inactive neurons increases during training, effectively reducing network capacity. This problem can be directly mitigated by periodically resetting a subset of the neurons~\cite{Frati_2024_reset_last, dormantSokar}. A second factor is optimizer instability~\cite{lyle_understand}, as standard optimizers like Adam~\cite{adam} and stochastic gradient descent (SGD) are designed for stationary data distributions. In continual learning, the optimizer state (e.g., momentum) becomes biased toward previous tasks; to address this, researchers have proposed periodic optimizer resets~\cite{nikishin2022primacy, asadi2023reset_optimizer} and layer normalization~\cite{lyle_understand} to alleviate plasticity loss. Prolonged training can also lead to a loss of curvature directions, characterized by rank collapse in the Hessian, which restricts directions available for future optimization~\cite{lewandowski2024directionscurvatureexplanationloss}. Regularization methods are proposed to prevent model weights from deviating from the initial values, for instance, L2 regularization~\cite{kumar2024maintaining, lyle2022understanding, Dohare2024}, Wasserstein regularization~\cite{lewandowski2024directionscurvatureexplanationloss}, and spectral regularization~\cite{lewandowski2025learning}. 

Additionally, loss of generalization in continual learning is often attributed to memorization. As training progresses, weight magnitudes increase, and the model tends to overfit--memorizing specific samples from previous tasks rather than learning generalizable features. The Shrink\&Perturb method~\cite{warm_adams} mitigates this by scaling down weight magnitudes and injecting noise to reduce memorization. It aims to force the model selectively forget memorized noise while preserving learned features. 

Moreover, recent work shows that carefully designed activation functions can effectively mitigate plasticity loss. Examples include Concatenated ReLUs (CReLUs; \citeauthor{abbaslp_RL}, \citeyear{abbaslp_RL}), Smooth-Leaky activation families~\cite{anonymous2026activationdesign}, and highly trainable deep Fourier features~\cite{lewandowski2025fourier}.

Several heuristic approaches have also been proposed, including smaller batch sizes~\cite{small_batch_ceron}, orthogonal parameter regularization~\cite{ortho_parseval_reg}, and reducing churn—unintended prediction drift induced by data outside the current batch~\cite{tang2025churn}.

It is important to notice that existing research on loss of plasticity assumes that the tasks change abruptly, which is less realistic in the real world. In our work, we investigate gradually changing environments, and show that loss of plasticity becomes less an issue. In addition, simulating a gradually changing environment by interpolation or task sampling can be a simple method that mitigates loss of plasticity (even if the tasks change abruptly).

\section{Problem Formulation}
\label{problem_formulation}

The problem of loss of plasticity has recently been observed in continual learning when the data distribution in each sequentially coming task is changed rapidly and non-stationarily. In the continual learning literature, the concept of loss of plasticity can refer to either reduced ability to learn new from new data (loss of trainability) \citep{dohare2021continual,lyle2022understanding,elsayed2024addressing,ma2024revisiting}, or reduced ability to generalize to unseen data (loss of generalizability) \citep{Dohare2024,warm_adams,slow_steady, mutual_frozen}. Our work investigates the loss of plasticity problem from the perspective of both trainability and generalizability.

Let $\mathcal{D} = \{(\bm x_m, y_m)\}_{m=1}^M$ be a training dataset where $\bm x_m \in \mathbb{R}^d$ is an input sample (e.g., an image with $d$ pixels) and $y_m \in \mathbb{Y}$ be the corresponding label. We denote a deep neural network by a function $f_\theta$. The learning algorithm optimizes the parameters by minimizing the loss of each task, given by $\operatorname{minimize}_\theta  \mathbb{E}_{(x, y)\sim p_t} [\ell(f_\theta(\bm x), y)]$. 
There are different settings in continual learning regarding the availability of samples in previous tasks. We adopt the common setting in the loss of plasticity research, where we assume the machine learning model is able to iterate over a task's samples multiple times in a mini-batch fashion, although the task boundary is not informed~\citep{minibatch, lewandowski2024directionscurvatureexplanationloss, lyle2022understanding}.

To further illustrate this process, we introduce two popular examples in the previous literature that simulate a continually changing environment.

\textbf{Random Image Labeling.} In this environment, the label of an image is randomly selected from $0,\cdots, c$ where $c$ is the number of classes in the dataset. In other words, the dataset of the $t$th task is $\mathcal{D}^{(t)}=\{(\bm x_m,y^{(t)}_m)\}_{m=1}^M$, where $\bm x_m$ comes from the original image dataset, and  $y_m^{(t)}\in\{0,\cdots, c\}$ is randomly sampled. Notice that the label remains fixed within a given task, but is reassigned for each new task.

\textbf{Random Pixel Permuting.} This environment permutes the image pixels of an image data sample. For each task, we sample a random permutation $\pi^{(t)}: \{1, \cdots, d \} \rightarrow \{1, \cdots, d \}$, where $d$ is the feature dimension (i.e., the number of pixels). The permutation $\pi^{(t)}$ is applied to all images for the $t$th task, which yields the dataset $\mathcal{D}^{(t)}=\{(\bm x_m^{(t)},y_m)\}_{m=1}^M$ for $ \bm x_{m,i}^{(t)}=\bm x_{m,\pi^{(t)}(i)}$. 

To further investigate loss of plasticity in the language models, we propose two text generation environments.

\textbf{Random Seq2Seq.} This environment utilizes synthetic text to evaluate the model's ability to learn arbitrary sequence mappings. The dataset for the $t$th task is denoted as $\mathcal{D}^{(t)}=\{(\bm x_m^{(t)}, \bm y_m^{(t)})\}_{m=1}^M$. Both the input $\bm x_m^{(t)}$ (of sequence length $p$) and the target $\bm y_m^{(t)}$ (of sequence length $q$) consist of synthetic ``words,'' where each word is a random character string of fixed length $L$ sampled from a standard alphabet. 

\textbf{Bigram Cipher.} To evaluate generalization plasticity, we must include certain learnable patterns between input and output, as opposed to the Random Seq2Seq task. To this end, we propose the Bigram Cipher environment with a dataset $\mathcal{D}^{(t)}=\{(\bm x_m^{(t)}, \bm y_m^{(t)})\}_{m=1}^M$. For each task $t$, we sample a random permutation $\pi^{(t)}: \mathcal{V} \rightarrow \{0, \cdots, |\mathcal{V}|-1\}$ over the vocabulary $\mathcal{V}$. The output tokens $\bm y_m^{(t)}$ is determined by the modular sum of the permuted values of the current ($\bm x_{m,i}^{(t)}$) and previous ($\bm x_{m,i}^{(t-1)}$) input tokens, where $i$ is the $i$th non-whitespace letter of $\bm x_m$. Given an input sequence of tokens $\bm x_m^{(t)} = (\bm x_i)_{i=1}^{Lp}$, the target sequence of tokens $y_m^{(t)}$ is determined by the sequence of modular sum of the permuted values of the current character $\bm x_{m,i}^{(t)}$ and the previous character $\bm x_{m,i}^{(t-1)}$:
    ${\pi^{(t)}}^{-1}(\bm y_{m,i}^{(t)}) = (\pi^{(t)}(\bm x_{m,i}^{(t)}) + \pi^{(t)}(\bm x_{m,(i-1)}^{(t)})) \pmod{|\mathcal{V}|}$.

\section{A Gradually Changing Environment}
While prior work has provided fundamental insights into the phenomenon of plasticity loss, it has primarily focused on settings with abrupt task transitions. However, natural environments and human cognition often evolve in a more gradual, continuous manner over time~\citep{mayr1970populations,gradualevolve}. Consequently, the discrete task boundaries used in existing benchmarks may not fully capture the dynamics of these incremental changes, leaving it an open question whether plasticity is similarly compromised in a gradually changing world.

In our work, we develop several methods to simulate a gradually changing environment. Consider the $t$th task represented by the dataset $\mathcal{D}^{(t)}$ and its subsequent task with dataset $\mathcal{D}^{(t+1)}$. The simulation of a gradually changing environment can be accomplished by different ways depending on the task nature.

\textbf{Output Interpolation.} For the Random Image Labeling  environments~\citep{dohare2021continual,lewandowski2024directionscurvatureexplanationloss, lyle2022understanding}, the inputs are the same for $\mathcal{D}^{(t)}$ and $\mathcal{D}^{(t+1)}$, but the outputs become new random labels. We perform label smoothing~\citep{smoothing_nips} for task interpolation. Let $\alpha_t\in [0,1]$ be an interpolation coefficient that starts from $0$ and increases uniformly by a constant step size $s$. We have $\bm {y}_m^{(\alpha_t)} = (1-2\alpha_t) \bm{y}_m^{(t)} + 2\alpha_t \bm u$ for $\alpha_t \in[0,\frac{1}{2}]$, where $\bm y_m^{(t)}$ is a one-hot vector for target $\bm y_m^{(t)}$ and $\bm u$ is a uniform distribution over the discrete target categories; for $\alpha_t \in [\frac{1}{2}, 1]$, we have $\bm {y}_m^{(\alpha_t)} = (2\alpha_t-1) \bm{y}_m^{(t+1)} + (2-2\alpha_t)\bm u$ and $\alpha_{t+1} = \alpha_t + s$. In other words, we first anneal $\mathcal D^{(t)}$ to a uniform distribution and then anneal it to $\mathcal D^{(t+1)}$, offering a smooth transition from the $t$th task to the $(t+1)$th task.

\textbf{Input Interpolation.} For the Random Image Permuting environment, the input pixel locations are shuffled in a new task, while the output remains the same. In this setting, we propose to gradually interpolate the input \textcolor{black}{(i.e., pixel values)} by direct averaging, i.e., $\bm x_m^{(\alpha)} = (1-\alpha) \bm x_m^{(t)} + \alpha \bm x_m^{(t+1)}$, \textcolor{black}{where $\bm x_m^{(t)}$ and $\bm x_m^{(t+1)}$ denote the $m$-th sample under the previous and new permutations, respectively.} As \textcolor{black}{before}, $\alpha$ is \textcolor{black}{an interpolation coefficient governed by the step size $s$, enabling a smooth transition from the previous input structure to the new one}.

It is important to note that input and output interpolations depend on the nature of the environment. In particular, they require the $m$th samples in $\mathcal D^{(t)}$ and $\mathcal D^{(t+1)}$ to have certain correspondence. This requirement can be reflected in certain real-world scenarios, such as changes in word meanings, where the input word is fixed while its associated meaning changes. However, such a correspondence may be unrealistic for other contexts, such as robots transitioning to a new environment, where the relationship between input and output is less predicable. To address this, we propose a general task interpolation method by annealed sampling.

\textbf{Task Sampling.} For more general task transitions, we sample a subset $\mathcal{D}^{(t)'}$ of size $\ceil{(1-\alpha) M}$ from $\mathcal{D}^{(t)}$ and a subset $\mathcal{D}^{(t+1)'}$ of size $\floor{\alpha M}$ from $\mathcal{D}^{(t+1)}$ to construct a new intermediate dataset $\mathcal{D}^{\alpha} = \mathcal{D}^{(t)'} \cup \mathcal{D}^{(t+1)'}$. Similar to input/output interpolation, $\alpha\in[0,1]$ is a coefficient that starts by $0$ and increments by a certain step size.

\textbf{Theoretical Analysis.} 
\label{theoretical_analysis}
We now give an explanation of why a gradually changing environment can mitigate the loss of plasticity. Intuitively, an abrupt task change causes abrupt change of the loss landscape, and thus the parameters may be trapped in a poor local minimum in the new landscape. On the other hand, if the task changes gradually, the loss landscape also evolves gradually. With proper gradient descent, the parameters may remain in the catchment basin corresponding to the original local minimum (which is optimized for the first task).  Thus, the loss of plasticity is much alleviated in such a scenario. 

We provide a formal analysis under two standard assumptions from optimization theory: smoothness and local strong convexity.  
Smoothness (Def.~\ref{def:smooth}) controls how quickly the gradient can change, preventing the landscape from exhibiting arbitrarily sharp curvature.
Local strong convexity (Def.~\ref{def:lsc}) ensures that each local minimum sits inside a basin with sufficient curvature, so that gradient descent can reliably converge to it.  
These two properties allow us to track how the basin geometry evolves for each small interpolation step on the loss function.

\begin{definition}[Smoothness]
\label{def:smooth}
Let $f: \mathbb{R}^d \rightarrow \mathbb{R}$ be a differentiable function. We say $f$ is $\beta$-smooth ($\beta>0$), if
\begin{align}
     \| \nabla f(\vx) - \nabla f(\vy) \| \le \beta \| \vx - \vy \| 
\end{align}
for any $\vx, \vy \in \sR^d$.
\end{definition}

\begin{definition}[Locally Strongly Convex]
\label{def:lsc}
For a (local) minimizer $\vx_f^*$ of a differentiable function $f$, we say $f$ is $(r,\mu)$-strongly convex on a convex set $\sD_{\vx^*_f} \supseteq \{ \vx \in \sR^d : \| \vx - \vx^*_f\| < r \}$, if \begin{align}
    f(\vx) \ge f(\vx_f^*) + \langle\nabla f(\vx_f^*), \vx_k-\vx_f^* \rangle+ \frac{\mu}{2}\| x-\vx_f^*\|^2
\end{align}
for any $\vx \in \sD_{\vx_f^*}$. Here, $r, \mu>0$.
\end{definition}

\textcolor{black}{Our first lemma establishes a basic building block that if the loss landscape is smooth and locally strongly convex in a basin, then GD with a sufficiently small step size will converge to the basin's minimizer without escaping it. }

\begin{lemma} [\textcolor{black}{GD Convergence in a Local Basin}]
\label{lemma:gd_converge}
    Consider gradient descent (GD) starting from any point in an $(r,\mu)$-locally strongly convex domain $\sD_{\vx_f^*}$ of a $\beta$-smooth function $f$, for some $\beta\ge\mu>0$. Let $(\vx_k)_{ k=1}^N$ be a sequence generated by GD. If the step size satisfies $\eta \le \min(\frac{1}{\beta}, \frac{r-\|\vx_k - \vx_f^*\|}{\beta \|\vx_k - \vx_f^*\|})$, we have (1) $\vx_k$ remains within the locally strongly convex region $\sD_{\vx_f^*}$ in every step $k$ of GD, and that (2) $\vx_k$ converges to $\vx_f^*$ as $k\rightarrow+\infty$. See proof in Appendix~\ref{appendix:gd_converge}.
\end{lemma}

The proof follows that for GD convergence in general, except that we need to check the radius of the locally strongly convex region. 

We then take a step further from proving convergence for a fixed loss function to a changing loss landscape. In our framework, a full task transition (e.g., switching from one random labeling to another) is not applied abruptly; instead, it is decomposed into a sequence of small interpolation steps. We model each such interpolation step (i.e., a \textit{micro-task}), as a near-identity linear transformation. The following lemma shows that a single micro-task preserves the essential basin structure.

\begin{lemma}[\textcolor{black}{Basin Property Preservation}]
\label{lemma:linear_smooth_convex}
Consider a gradually changing environment where, after one micro-task change, the loss function evolves from $f$ to $g(\vx) = f(\mW \vx)$, with $\mW \in \sR^{d \times d}$ full-rank. Assume $\mW$ is near-identity in the sense that
\begin{align}
    \|I - \mW\| \le \frac{\epsilon}{1+\epsilon}, \quad \frac{1}{1+\epsilon} \le \sigma_i(\mW) \le \frac{1}{1-\epsilon}, \quad \forall\, i,
\end{align}
for some $\epsilon > 0$. If $f$ is $\beta$-smooth and $(r,\mu)$-strongly convex in a local basin $\sD_{\vx^*_f}$, then $g$ is also smooth and locally strongly convex in $\sD_{\vx^*_g} = \{\,\vx : \mW \vx \in \sD_{\vx^*_f}\,\}$, with $\beta' = \beta(1+\epsilon)^2$, $\mu' = \mu(1+\epsilon)^{-2}$, and $r' = r(1-\epsilon)$. See proof in Appendix~\ref{appendix:linear_smooth_convex}.
\end{lemma}

\textcolor{black}{We can then verify that the GD iterates optimizing for the old loss function $f$ is actually lies inside the new locally strongly convex basin in $g$.}

\begin{lemma} [\textcolor{black}{Basin Containment}]
\label{lemma:ball_in_new_bowl} Under on the assumptions in Lemma~\ref{lemma:linear_smooth_convex}, the new locally convex set $\sD_{\vx^*_g}$ contains $\vx$ if $\sD_{\vx^*_f}$ contains $\mW\vx$. See proof in Appendix~\ref{appendix:ball_in_new_bowl}.
\end{lemma}

\textcolor{black}{With all three building blocks in place, we can now demonstrate that GD-optimized parameters would converge to a new local optimum of $g$ corresponding to the previous one, i.e., $\vx_g^*=\mW^{-1}\vx_f^*$.}

\begin{lemma} [\textcolor{black}{Convergence After Single Micro-task Transition}]
\label{lemma:converge_to_shifted_minimizer}
    For any $0 < c < r$ and learning rate $\eta \le \min(\frac{1}{\beta}, \frac{r-\|\vx_k - \vx_f^*\|}{\beta \|\vx_k - \vx_f^*\|})$, the GD algorithm converges to the new minimizer $\vx_g^*=\mW^{-1}\vx_f^*$, as long as (1) $\frac 1{1+\epsilon} \le \sigma_i(\mW) \le \frac 1 {1-\epsilon}$ with $\epsilon = \frac{r - c}{r - c + \| \vx_f^* \|}$ for all $i \in [d]$; and (2) the \textbf{current point} $\vx$ is close to the previous minimizer $x_f^*$ satisfying $\| \vx - \vx_f^* \| \le c (1-\epsilon)$. See proof in Appendix~\ref{appendix:converge_to_shifted_minimizer}. 
\end{lemma}

\textcolor{black}{Finally, we formally state our main theorem by extending Lemma~\ref{lemma:converge_to_shifted_minimizer} from a single micro-task to a sequence of micro-tasks via induction. The key observation is that after each micro-task, GD converges to the new basin's minimizer, restoring the preconditions needed for the next micro-task. The theorem implies that plasticity is preserved across arbitrarily large task changes---provided each interpolation step is sufficiently small.}

\begin{theorem}
\label{theorem}
    Let $f_0$ be $\beta_0$-smooth and $(r_0,\mu_0)$-strongly convex. Let $\vx_0^*$ be a local minimizer of $f_0$. We assume, under the gradually changing environment, the sequence of loss functions has the form of $f_t(\vx_t) = f_{t-1}(\mW_t\vx_{t-1})$ for $t \ge 1$, where $\mW_t$ satisfies $\| \sigma_i(\mW_t) - 1\| \le \frac{r-c+\|\vx_{t-1}^*\|}{\|\vx_{t-1}^*\|}, \forall i$, and \textcolor{black}{$\|\mI - \mW_t\| \le \frac{r-c}{2(r-c)+\|\vx_{t-1}^*\|}$}, for some $c\in(0,r)$. Consider $(\vx_{t,k})_{ k \in [N]}$ being a sequence generated by the GD algorithm while minimizing $f_t$ w.r.t. $\vx_t$, a learning rate $\eta \le \min(\frac{1}{\beta_t}, \frac{r_t-\|\vx_{t,k} - \vx_{t-1}^*\|}{\beta_t \|\vx_{t,k} - \vx_{t-1}^*\|})$ of the $t$th micro-task guaranties that $\vx_{t,k}$ converges to $\vx_t^*=\mW_t^{-1}\vx_{t-1}^*$. See proof in Appendix~\ref{appendix:theorem}.
\end{theorem}

In summary, our analysis provides an explanation on why a gradually changing environment retains plasticity, as the parameters learned for the previous task can be optimized to the corresponding optimum for the new task.

\section{Experiments}
\label{sec:experiments}

In this section, we empirically examine whether a gradually changing environment preserves plasticity in continual learning. We compare the following two scenarios:
(1) A gradually changing environment, and
(2) An abruptly changing environment with and without loss-of-plasticity mitigation methods. In particular, we consider widely adopted and recent methods as below:
\begin{compactitem}
\item \textbf{L2 regularization} \citep{l2}, penalizing the squared magnitude of weights; 
\item \textbf{Shrink\&Perturb} \citep{warm_adams}, which shrinks the model weights by a constant factor and adds noise; 
\item \textbf{Spectral regularization} \citep{lewandowski2025learning}, which penalizes the singular values of weight matrices; and
\item Recycling dormant neurons (\textbf{ReDO}; \citeauthor{dormantSokar}, \citeyear{dormantSokar}), which periodically re-initializes inactive neurons.
\end{compactitem}

\textcolor{black}{We adopt the same benchmarks (Random Image Labeling and Random Pixel Permuting) widely used in prior plasticity research~\citep{dohare2021continual, Dohare2024, lyle_understand, lewandowski2024directionscurvatureexplanationloss, lewandowski2025learning} to ensure that any difference in observed plasticity is attributable to the transition abruptness rather than task complexity. We also introduce two language tasks (Random Seq2Seq and Bigram Cipher) to extend the study of plasticity to sequence-to-sequence models, a setting largely unexplored in prior work despite the growing prominence of language models.}

As mentioned in \S\ref{problem_formulation}, the loss of plasticity may happen in the sense of both trainability~\citep{dohare2021continual,abbaslp_RL} and generalizability~\citep{warm_adams, slow_steady}; they are evaluated in \S\ref{subsec:trainability} and \S\ref{subsec:generalizability}, respectively. Subsection~\ref{subsec:analysis} provides additional in-depth analyses.

\begin{figure*}[!t]
  \begin{center}
    \centerline{
    \includegraphics[width=\linewidth]{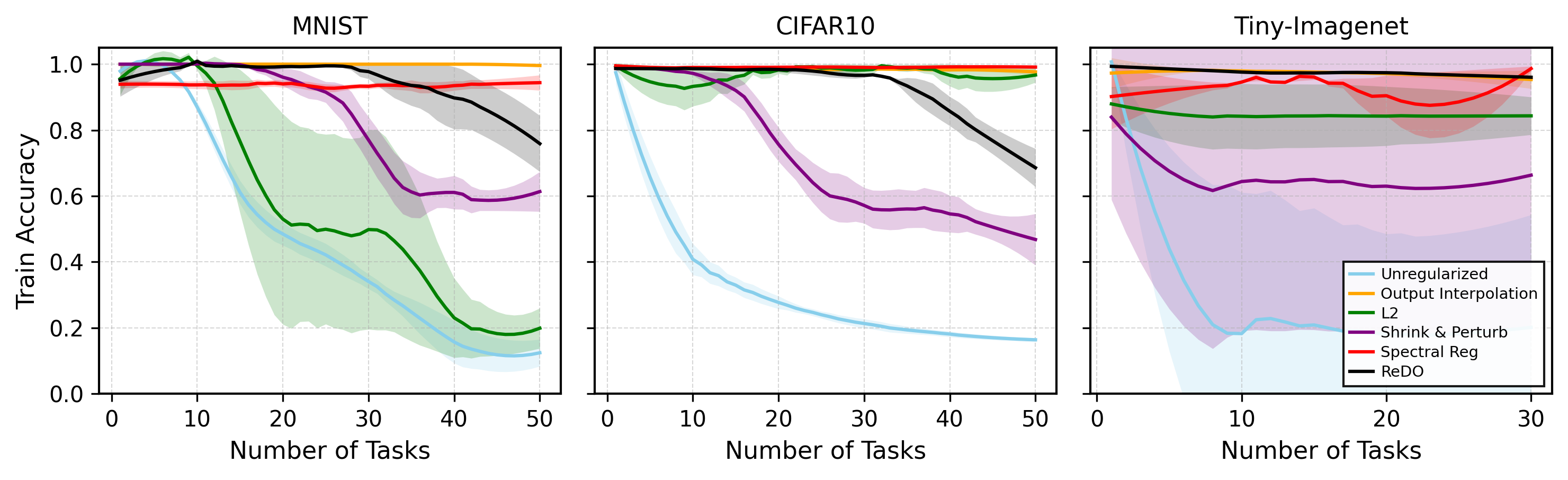}}
    \caption{\textbf{Trainability for Random Image Labeling tasks on MNIST and CIFAR10 using an MLP or a ResNet-18 model.} Output interpolation is more effective than other plasticity mitigation methods for these vision benchmarks. }
    \label{fig:rq1_random_label} 
  \end{center}
\end{figure*}

\subsection{Evaluation of Trainability}
\label{subsec:trainability}

Trainability refers to whether the model parameters can be optimized towards new tasks, and it is a fundamental requirement for continual learning \citep{lyle_understand,lewandowski2025learning}. Loss of trainability refers to a progressive decline in a model’s ability to learn new tasks. Extensive research on this phenomenon has examined the conditions under which neural networks fail to retain trainability in non-stationary environments~\citep{lyle2022understanding,elsayed2024addressing,Dohare2024, ma2024revisiting}.

\textcolor{black}{\textbf{Tasks and Models.} We experiment with the Random Image Labeling and Random Seq2seq environments (task formulations presented in \S\ref{problem_formulation}).}
In particular, we consider three image benchmark datasets with increasing scales: MNIST~\citep{mnist}, CIFAR-10~\citep{he2016deep}, and Tiny-ImageNet~\citep{le2015tiny}. Most previous plasticity studies have used selected datasets for their evaluation of trainability, but we include all three in our work.

\begin{figure}[!t]
    \includegraphics[width=\linewidth]{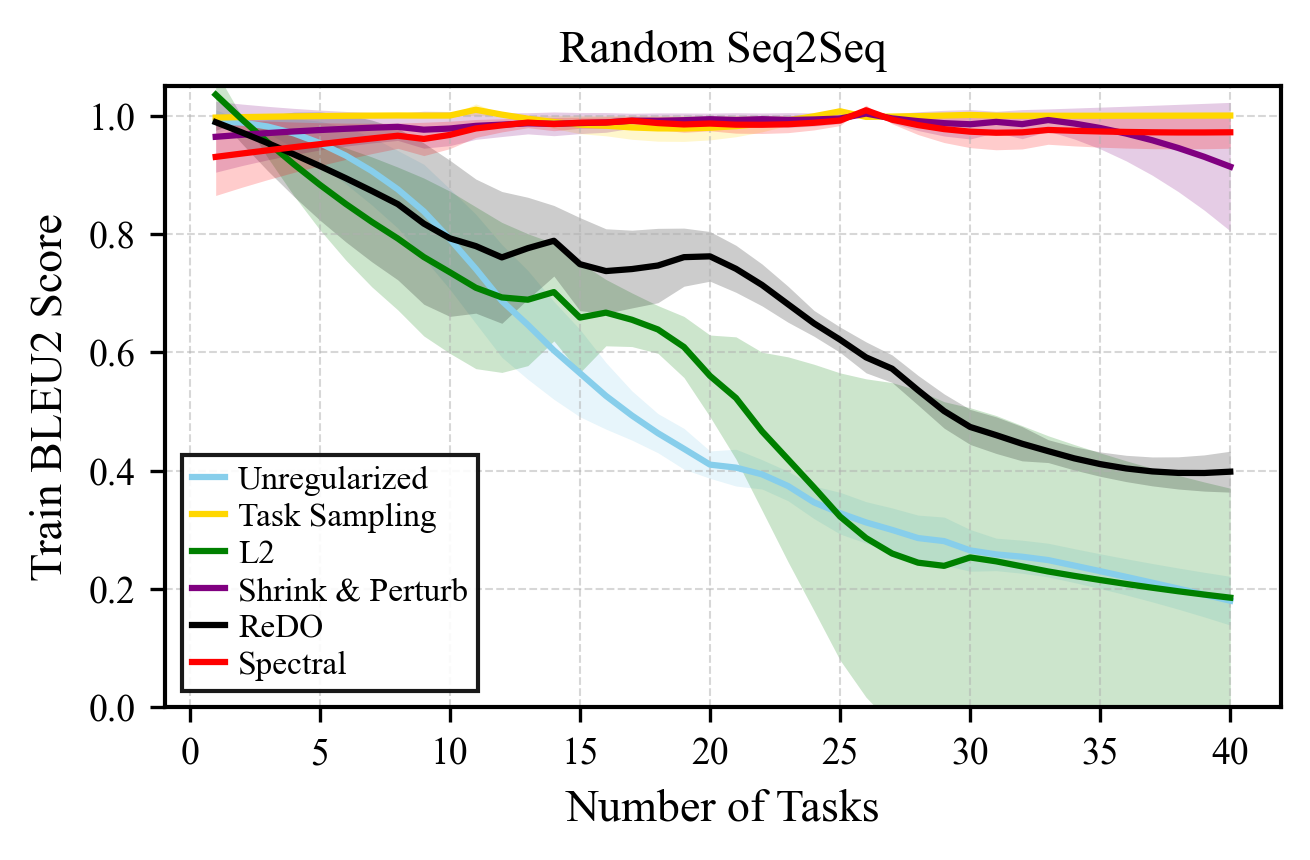} \caption{\textbf{Trainability for random Seq2Seq task on synthetic text using T5-small.} \textcolor{black}{Without any mitigators, T5-small has diminishing train BLEU2 score, suggesting loss of trainability. Task sampling is the most effective among all baselines in preserving plasticity.}} 
  \label{fig:random_seq2seq} 
\end{figure}

For MNIST, we utilize a 4-layer Multi-Layer Perceptron (MLP) with 256 units per layer. Each hidden layer consists of a linear transformation followed by Layer Normalization and a ReLU nonlinearity. For the more complex CIFAR-10 and Tiny-ImageNet tasks, we employ an off-the-shelf ResNet-18 \citep{he2016deep} architecture comprising four cascaded residual blocks with batch normalization and skip connections, totaling approximately 11 million parameters.

For Random Seq2Seq, we adopt a pre-trained T5-small model \citep{raffel2020exploring} but only keep the first Transformer block. Note that Transformers are powerful models; however, our Random Seq2Seq is synthetic and relatively simple. Therefore, we restrict the Transformer capacity to study the phenomenon of loss of plasticity in our research environment.

\textbf{Results.} 
Figs.~\ref{fig:rq1_random_label} and~\ref{fig:random_seq2seq} present the trainability results for aforementioned benchmarks. Results are averaged over 5 seeds; shaded regions span the minimum and maximum across seeds.

We observe loss of trainability in all the tasks. For example, in an unregularized MLP on the Random MNIST benchmark, the training accuracy collapses from 100\% to random guessing after just 50 tasks. This is consistent with findings in the previous work that neural networks lose plasticity with abruptly changing data distribution
~\citep{dohare2021continual,Dohare2024,lyle2022understanding,elsayed2024addressing,ma2024revisiting}. 

We compare standard mitigation strategies designed for abrupt changes against our proposed simulation of a gradually changing environment.

Among the mitigation strategies, performance varies significantly by the architecture and benchmarks. Shrink\&Perturb~\citep{warm_adams} and Spectral Regularization~\citep{lewandowski2025learning} excel only in the Seq2Seq language task, showing high volatility in vision benchmarks, particularly on complex datasets like Tiny-Imagenet. This disparity aligns with the findings by~\citeauthor{transformer_robust}, who suggest that Transformer models are inherently more robust to perturbations than CNNs \citep{transformer_robust}. L2 regularizated ResNet-18 sustains near perfect training accuracy for both larger vision benchmarls. However, the performance of L2 regularized MLP drops as quickly as an unregularized network, possibly due to dormant neurons and lack of skip connections. ReDO provides moderate remedy for plasticity loss in vision benchmarks via neuronal recycling; however, its performance on language tasks is less satisfactory. Previous work suggests that ReDO is ill-suited for Transformers, as the smooth nature of GELU/SiLU activations and the centering effects of Layer Normalization prevent the ``dying unit'' phenomenon that is common in ReLU-based CNNs~\citep{jain2026revive}. This phenomenon occurs when a neuron consistently outputs zero for all inputs, causing its gradient to vanish and preventing any further weight updates.

Our work investigates how a gradually changing environment affect plasticity. We find that a gradually changing environment---either simulated by output interpolation or task sampling---largely preserves the plasticity of neural networks. Output interpolation effectively preserves trainability in vision benchmarks, while task sampling is effective for language domains, matching or exceeding the performance of competing strategies. 
Admittedly, our simulation of gradually changing environments is not perfect, as we observe minor plasticity loss after task 35 for Random MNIST. That said, our approach still outperforms the best baseline from prior work. These results suggest that simulating a gradually changing environment is a simple yet theoretically grounded mitigation for the loss of trainability.

In summary, mitigation methods relying on resetting, regularization, or noise injection offer only partial mitigation; their strict global constraints on the optimization landscape lead to high instability. Further, these methods necessitate exhaustive hyperparameter tuning for each distinct environment; without precise calibration, their constraints often impede optimization. In contrast, our simulation of a gradually changing environment provides a robust, model-agnostic solution that preserves plasticity without requiring extensive hyperparameter search.

\subsection{Evaluation of Generalizability}
\label{subsec:generalizability}

Loss of generalizability refers to a model’s declined ability to apply learned knowledge to unseen data. Since trainability is a prerequisite for generalization, prior work has largely focused on the former \citep{dohare2021continual,lyle2022understanding,dormantSokar,lewandowski2024directionscurvatureexplanationloss,lewandowski2025learning}. However, because high training accuracy can be achieved through memorization, it is critical to evaluate generalizability to confirm that the model has learned underlying patterns rather than merely overfitting the training set.

\textcolor{black}{\textbf{Tasks and Models.} We investigate with the Random Pixel Permuting and Random Bigram Cipher environments. These permutation-based formulations, as detailed in \S\ref{problem_formulation}, are particularly well-suited for evaluating generalizability because they can be applied directly to the test splits of the benchmarks.}

\begin{figure}[!t]
    \centering
    \begin{minipage}[t]{0.5\textwidth}
        \centering
        \includegraphics[width=\linewidth]{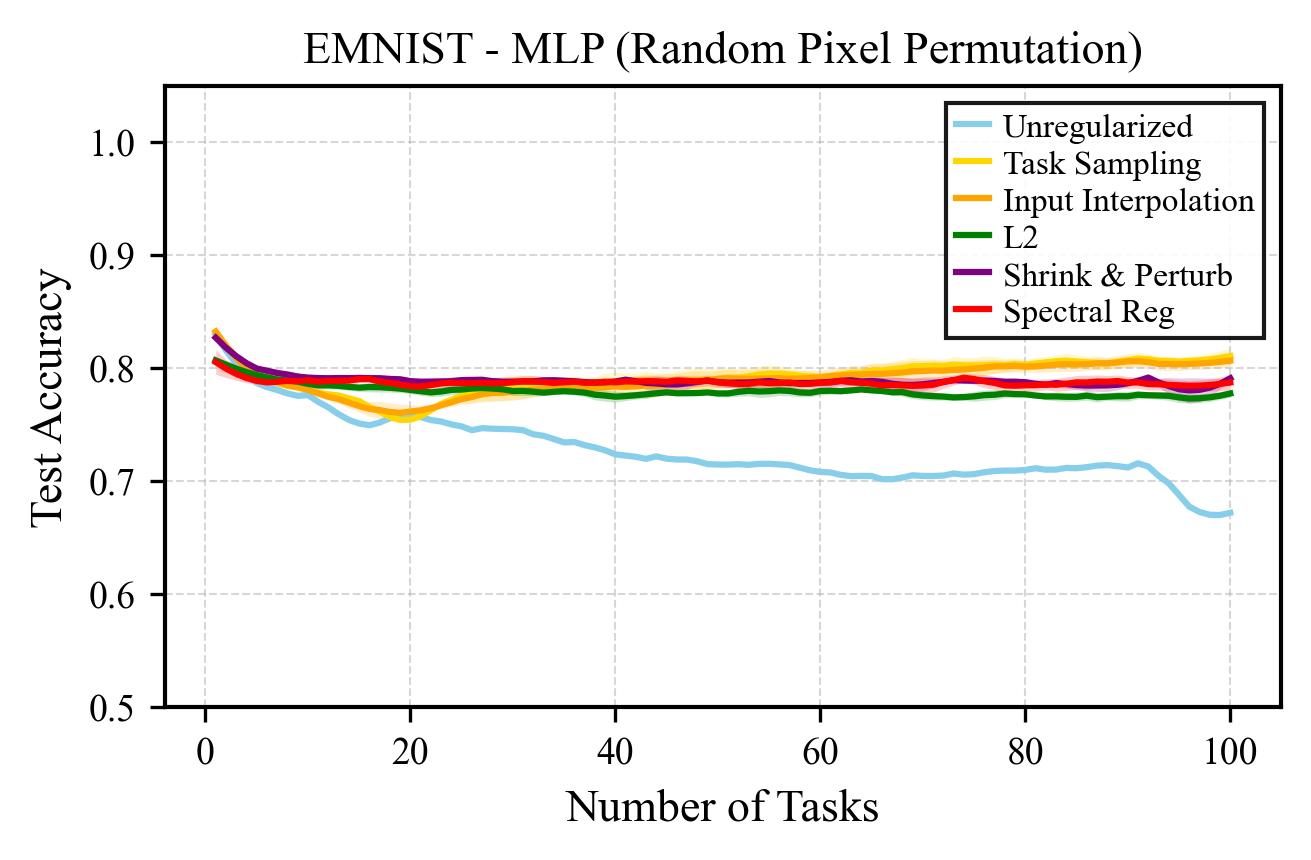}
        \captionof{figure}{\textbf{Continual learning with Random Pixel Permuting tasks on EMNIST using a 4-layer MLP model.} \textcolor{black}{Plasticity loss is mild in this setting, yet input interpolation and task sampling still achieve the highest test accuracy among all methods.}}
        \label{fig:rq2_pixel_permute}
    \end{minipage}
    \hfill
    \begin{minipage}[t]{0.5\textwidth}
        \centering
        \includegraphics[width=\linewidth]{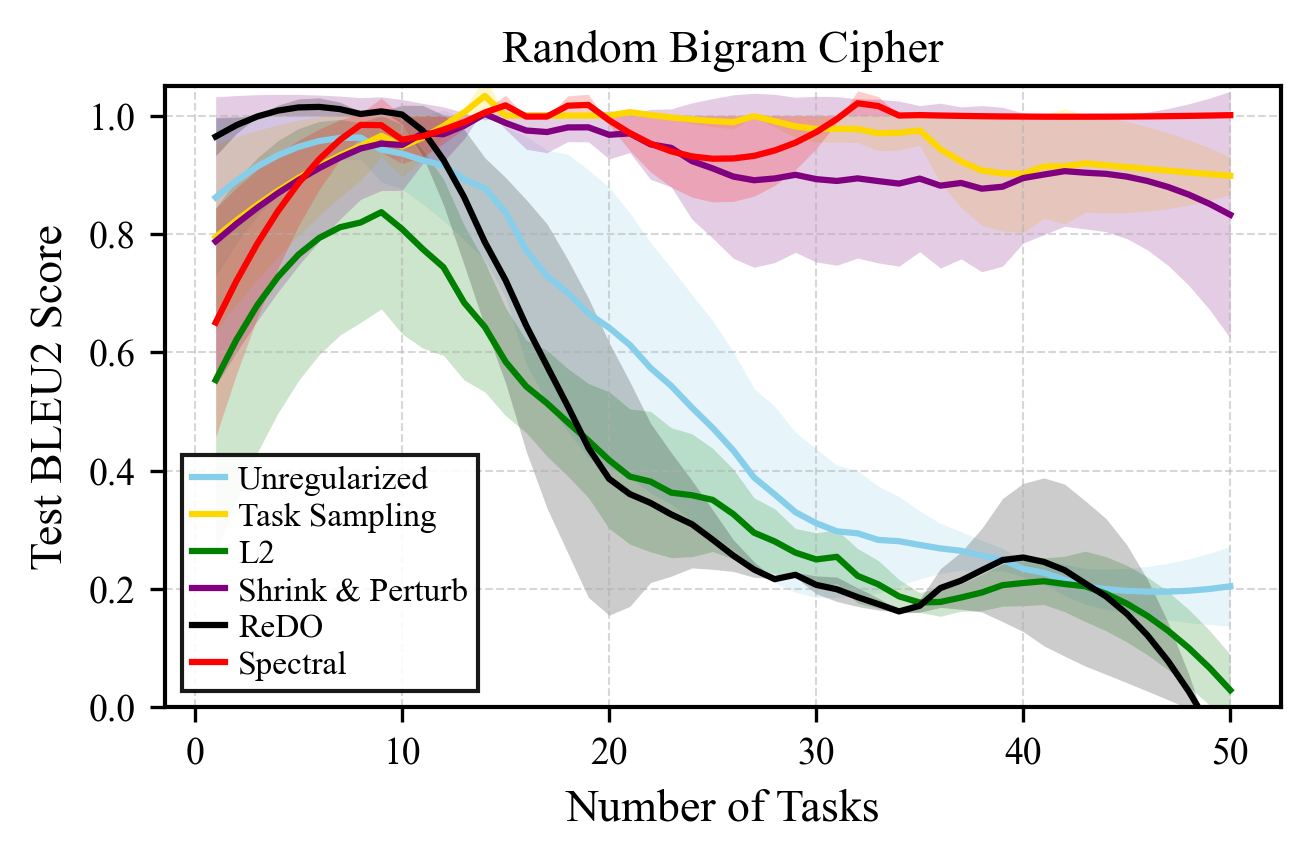}
        \captionof{figure}{\textbf{Generalizability evaluated by test BLEU2 score on Bigram Cipher tasks on customized T5-small model.} \textcolor{black}{Task sampling and Spectral Regularization maintain strong generalizability, while the unregularized baseline, L2, and ReDO degrade significantly.}}
        \label{fig:bigram_cipher}
    \end{minipage}
\end{figure}

\textbf{Results.} Figs.~\ref{fig:rq2_pixel_permute} and~\ref{fig:bigram_cipher} demonstrate the generalizability results for the benchmarks described above. 

For Random Pixel Permuting, we observe only a mild plasticity loss (test accuracy drops around 10\%) for the unregularized MLP on the Random EMNIST benchmark with 100 tasks.  Such a mild plasticity loss is consistent with previous work~\citep{kumar2024maintaining, lewandowski2025learning}, and is likely because the nature of the task allows the model to reuse previously learned image features, thereby facilitating adaptation. In this task, we observe all previously introduced mitigation methods are able to preserve plasticity to a large extent (Fig.~\ref{fig:rq2_pixel_permute}). Nevertheless, our task sampling and input interpolation (yellow and orange curves) remain the two most competitive methods at the 100th task.

In the Random Bigram Cipher task (Fig.~\ref{fig:bigram_cipher}), we observe a drastic loss of generalizability for an unregularized T5-small model, with the test BLEU2 score falling to 20\% after 50 tasks, reflecting high complexity of the task. Consistent with our Random Seq2Seq findings, ReDO and L2 regularization fail to retain generalizability in the language domain. Similarly, Spectral Regularization and Shrink\&Perturb successfully preserve generalizability, mirroring their strong performance on the Random Seq2Seq task. 

Overall, our proposed simulation of a gradually changing environment preserves generalizability in both the vision task and the complex language task, matching the performance of the best-performing mitigation methods designed for abrupt task changes. Experiment details are shown in Appendix~\ref{appendix:B}.

\subsection{In-Depth Analyses}
\label{subsec:analysis}

\begin{figure*}[!t] 
    \centering
    \begin{subfigure}[b]{0.44\textwidth}
    \centering
    \includegraphics[width=\linewidth]{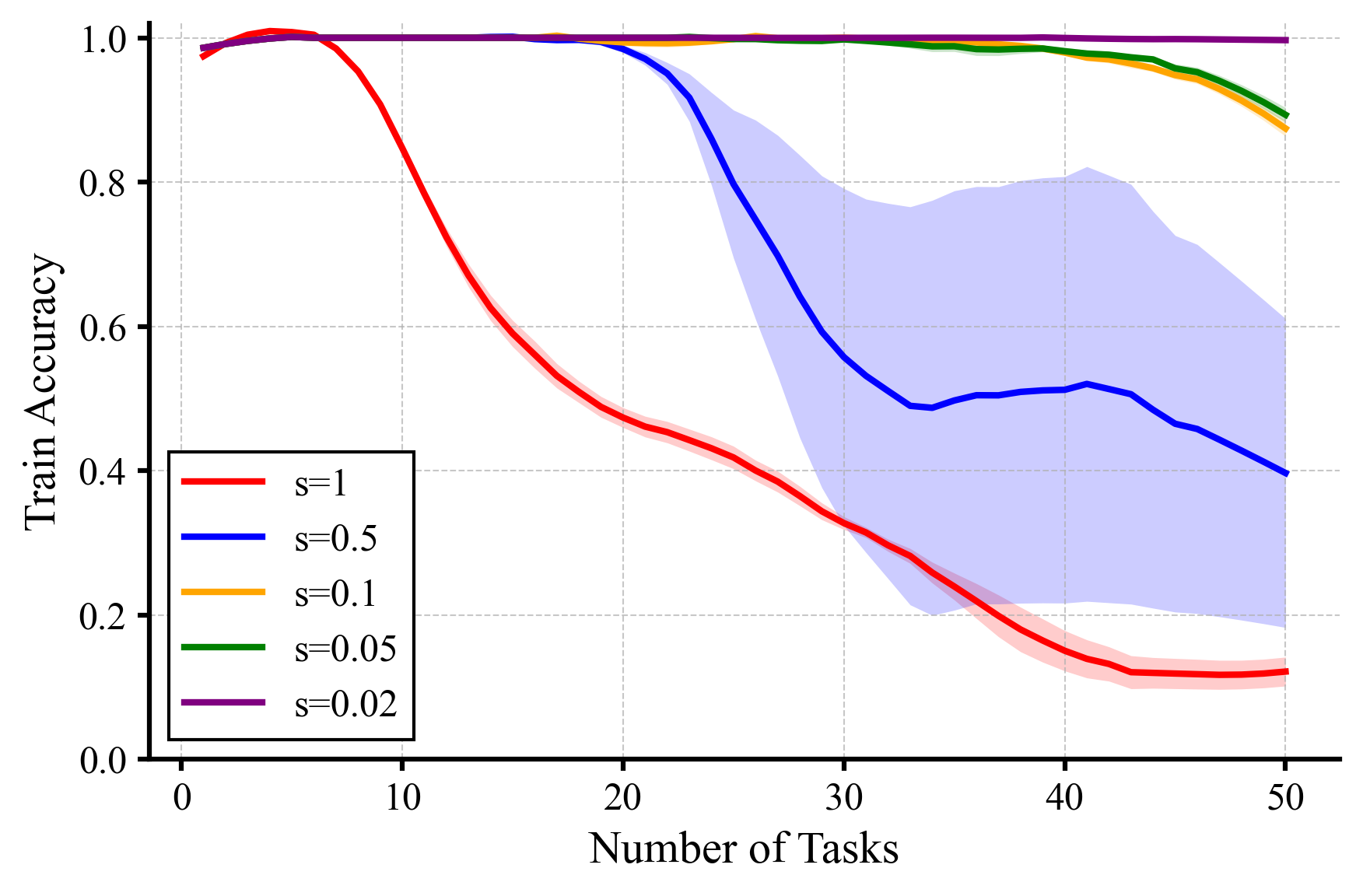}
        \caption{Random MNIST Output Interpolation}
        \label{fig:mnist_granularity}
    \end{subfigure}
    \hfill 
    \begin{subfigure}[b]{0.44\textwidth}
        \centering
        \includegraphics[width=\linewidth]{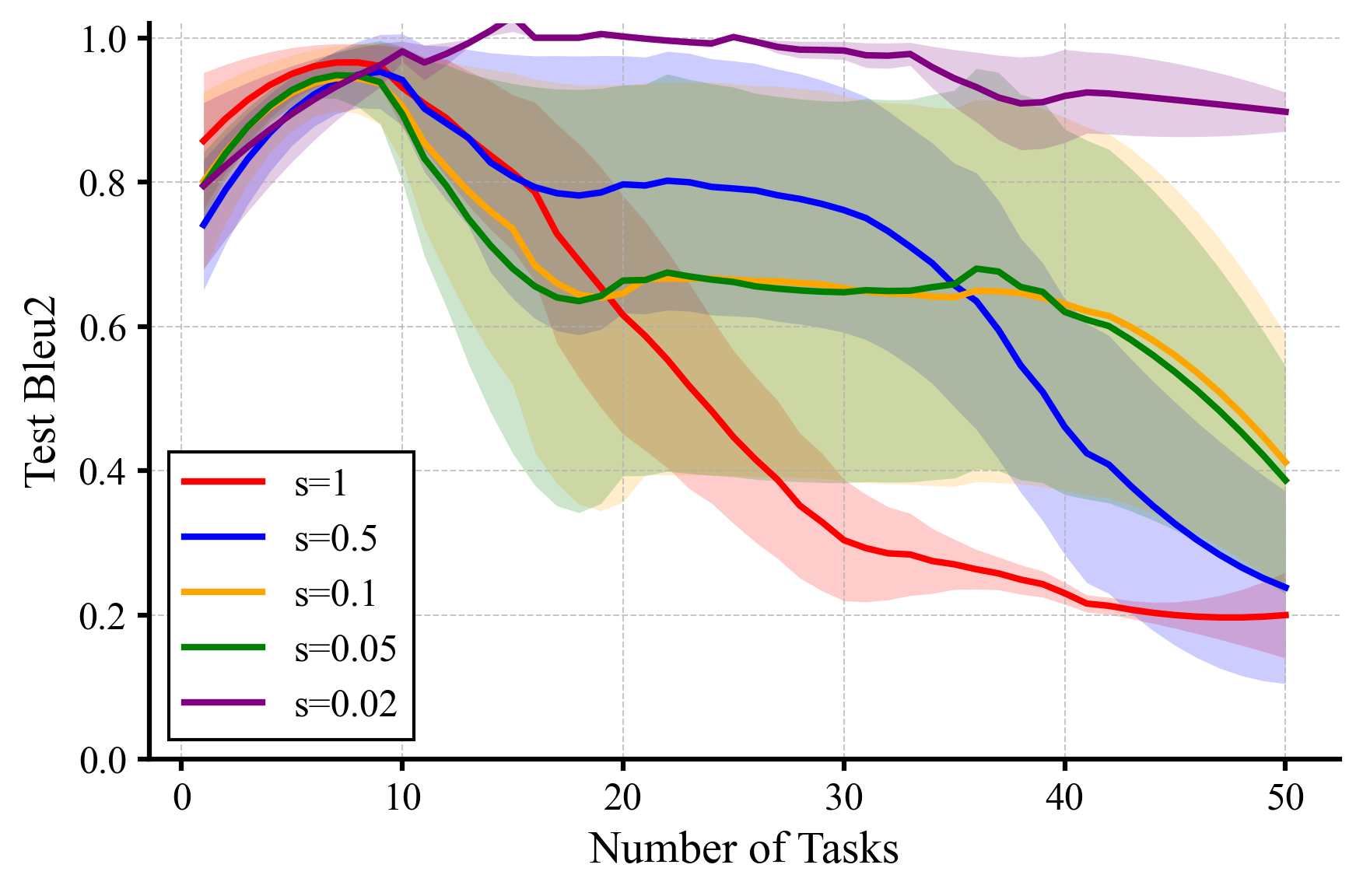}
        \caption{Bigram Cipher Task Sampling}
    \label{fig:cipher_granularity}
    \end{subfigure}
    
    \caption{\textbf{The effect of granularity of the interpolation step size on plasticity preseivation for both trainability and generalizability task.} \textcolor{black}{Abrupt task change ($s\!=\!1$) causes severe plasticity loss in both settings. As the step size decreases, plasticity is progressively better preserved, with the finest granularity ($s\!=\!0.02$) maintaining stable performance throughout. This confirms a monotonic relationship between transition smoothness and plasticity retention.}}
    \label{fig:combined_granularity}
\end{figure*}

\begin{figure*}[t]
    \centering
    \includegraphics[width=0.77\linewidth]{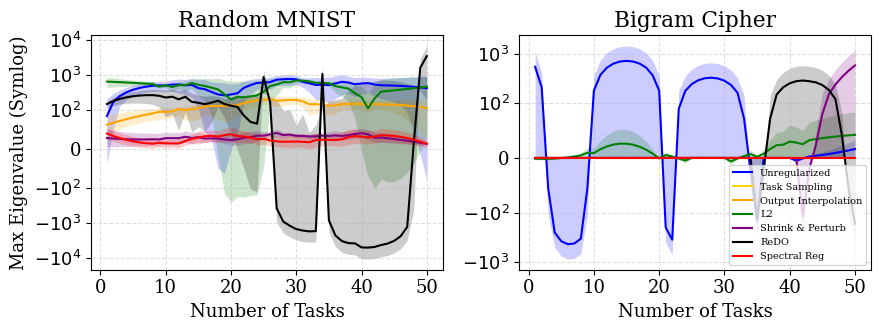}
    \caption{\textbf{Maximum eigenvalue for Random MNIST and Bigram Cipher.}
    Interpolation and task sampling keep the maximum eigenvalue lower across tasks, suggesting a less sharp loss landscape under gradual task transitions.}
    \label{fig:landscape_analysis}
\end{figure*}

\textbf{Granularity of Interpolation and Sampling.} To investigate the relationship between task transition smoothness and plasticity, we adopt two distinct experimental setups: the classic MNIST model with a 4-layer MLP model due to limited resources, and the Bigram Cipher setup with a T5-small model for generalizability. We adjust the interpolation coefficient $\alpha$ (as mentioned in \S\ref{problem_formulation}) from $0$ to $1$ with a step size ranging from $\{0.02, 0.05, 0.1, 0.5\}$. As seen in Figs.~\ref{fig:mnist_granularity} and~\ref{fig:cipher_granularity}, abrupt task change (step size $=1$) causes significant accuracy drop after $20$ tasks for Random MNIST and $12$ tasks for Bigram Cipher, respectively, while coarse-grained task change (step size $=0.5$) mitigate these declines but are still not satisfactory. In comparison, step size $0.02$ sustains near-perfect accuracy throughout the experiments. \textcolor{black}{We have also investigated the relationship between interpolation granularity and four neural network properties including maximum hessian eigenvalue, relative representation change, dormant unit fraction, and distance from initialization (shown in Appendix~\ref{appendix:C})}. Overall, this analysis verifies that a gradually changing environment does preserve plasticity, and that a finer-grained task transition preserves more plasticity.

\textbf{Loss Landscape Geometry.} To empirically characterize why gradually changing environments retain plasticity, we analyze from the loss landscape perspective. Previous work has shown that the maximum eigenvalue of the weight matrices can be interpreted as the sharpness of the loss landscape~\citep{dinh_sharpness, lyle_understand, lewandowski2024directionscurvatureexplanationloss}.  We evaluate the loss landscape geometry on a trainability task (Random MNIST) and a generalizability task (Bigram Cipher) with this metric.  As shown in Fig.~\ref{fig:landscape_analysis}, top-performing methods including Spectral Regularization and Shrink\&Perturb have relatively small eigenvalues. Notably, gradual task transitions (through interpolation and task sampling) naturally maintain low and stable eigenvalues without explicit regularization penalties, suggesting a smooth landscape for optimization.

\section{Conclusion}
Prior plasticity research has largely focused on an artificial setup of abrupt task change, which may not fully reflect real-world continual learning scenarios. In this paper, we emphasize that real-world environments often evolve gradually. We show both theoretically and empirically that plasticity can be preserved in a gradually changing environment. Even for abruptly changing setups, loss of plasticity can be effectively mitigated by simulating gradual change from task to task through input/output interpolation and task sampling.

\bibliography{example_paper}
\bibliographystyle{icml2026}

\newpage
\appendix
\onecolumn
\section{Proofs of Lemmas and Theorem}
\label{appendix:A}

For completeness, we restate the descent lemma~\citep{descent_lemma} below. 
\begin{lemma}[Descent Lemma]
\label{lemma:descent}
Let $f: \sR^d \to \sR$ be a differentiable function. 
If $f$ is $\beta$-smooth, then for all $\vx, \vy \in \sR^d$,
\begin{align}
    f(\vy) \le f(\vx) + \langle \nabla f(\vx), \vy - \vx \rangle + \frac{\beta}{2}\|\vy - \vx\|^2.
\end{align}
\end{lemma}

\begin{proof}
Let $\vh = \vy - \vx$ and define the scalar function $\phi(t) = f(\vx + t\vh)$ for $t \in [0, 1]$. 
By the fundamental theorem of calculus,
\begin{align}
    f(\vy) - f(\vx) = \int_0^1 \nabla f(\vx + t\vh)^\top \vh \, dt.
\end{align}
Subtracting and adding $\nabla f(\vx)^\top \vh$ inside the integral gives
\begin{align}
    f(\vy) - f(\vx) - \nabla f(\vx)^\top \vh 
    = \int_0^1 [\nabla f(\vx + t\vh) - \nabla f(\vx)]^\top \vh \, dt.
\end{align}
Since $f$ is $\beta$-smooth, we have 
$\|\nabla f(\vx + t\vh) - \nabla f(\vx)\| \le \beta t \|\vh\|$. 
Applying the Cauchy–Schwarz inequality and integrating over $t \in [0, 1]$ yields
\begin{align}
    f(\vy) - f(\vx) - \nabla f(\vx)^\top \vh 
    \le \int_0^1 \beta t \|\vh\|^2 \, dt = \frac{\beta}{2}\|\vh\|^2.
\end{align}
Rearranging the terms gives the desired result.
\end{proof}

\subsection{Proof of Lemma~\ref{lemma:gd_converge}}
\label{appendix:gd_converge}

\begin{proof}
Since $\vx_f^*$ is a local minimizer, $\nabla f(\vx_f^*) = 0$. 

By $\beta$-smoothness (Def.~\ref{def:smooth}),

\begin{align}
\label{eq:grad_bound}
\|\nabla f(\vx_k)\| 
= \|\nabla f(\vx_k) - \nabla f(\vx_f^*)\| 
\le \beta \|\vx_k - \vx_f^*\|.
\end{align}

\textbf{Part (1): Iterates remain inside $\sD_{\vx_f^*}$.}

From the GD update 
$\vx_{k+1} = \vx_k - \eta \nabla f(\vx_k)$ and 
\eqref{eq:grad_bound},
\begin{align}
\label{eq:triangle}
\|\vx_{k+1} - \vx_f^*\| 
\le \|\vx_k - \vx_f^*\| + \eta\beta\|\vx_k - \vx_f^*\|.
\end{align}

\textcolor{black}{If 
$\eta \le \frac{r - \|\vx_k - \vx_f^*\|}
{\beta\|\vx_k - \vx_f^*\|}$ is the binding constraint, 
then $\eta\beta\|\vx_k - \vx_f^*\| 
\le r - \|\vx_k - \vx_f^*\|$, so 
\eqref{eq:triangle} gives 
$\|\vx_{k+1} - \vx_f^*\| \le r$.  
\\
If instead $\eta \le 1/\beta$, then  $1/\beta \le \frac{r - \|\vx_k - \vx_f^*\|} {\beta\|\vx_k - \vx_f^*\|}$, which implies $\|\vx_k - \vx_f^*\| \le r/2$, so \eqref{eq:triangle} gives $\|\vx_{k+1} - \vx_f^*\| \le 2\|\vx_k - \vx_f^*\| \le r$. 
\\
In both cases $\vx_{k+1} \in \sD_{\vx_f^*}$, and the claim follows by induction.}

\textbf{Part (2): Geometric convergence.}

\textcolor{black}{By the descent lemma 
(Lemma~\ref{lemma:descent}) with $\vy = \vx_k - \eta\nabla f(\vx_k)$ and $\eta \le 1/\beta$:
\\
\begin{align}
f(\vx_{k+1}) 
\le f(\vx_k) 
- \frac{\eta}{2}\|\nabla f(\vx_k)\|^2.
\end{align}
\\
From local strong convexity (Def.~\ref{def:lsc}) and $\nabla f(\vx_f^*) = 0$, we have $f(\vx_k) - f(\vx_f^*) \ge \frac{\mu}{2}\|\vx_k - \vx_f^*\|^2$, which combined with~\eqref{eq:grad_bound} yields $\|\nabla f(\vx_k)\|^2 \ge 2\mu\bigl(f(\vx_k) - f(\vx_f^*)\bigr)$. 
\\
By substituting we have:
\\
\begin{align}
f(\vx_{k+1}) - f(\vx_f^*) 
\le (1 - \eta\mu)
\bigl(f(\vx_k) - f(\vx_f^*)\bigr).
\end{align}
\\
Since $\eta\mu \le \mu/\beta \le 1$, the function values converge geometrically, which implies geometric decay and $\vx_k \to \vx_f^*$ as $k\!\to\!\infty$.}
\end{proof}

\subsection{Proof of Lemma~\ref{lemma:linear_smooth_convex}}
\label{appendix:linear_smooth_convex}

\begin{proof}
The gradient of $g$ is $\nabla g(\vx) = \mW^\top \nabla f(\mW\vx)$.  
For any $\vx, \vy \in \sR^d$,
\begin{align}
\|\nabla g(\vx) - \nabla g(\vy)\| 
= \|\mW^\top (\nabla f(\mW\vx) - \nabla f(\mW\vy))\|
\le \|\mW\|^2 \cdot \beta \|\vx - \vy\|
\le \beta(1+\epsilon)^2 \|\vx - \vy\|,
\end{align}
so $g$ is $\beta' = \beta(1+\epsilon)^2$-smooth.

Next, since $f$ is $\mu$-strongly convex,
\begin{align}
f(\mW\vx) \ge f(\mW\vy) 
+ \nabla f(\mW\vy)^\top (\mW\vx - \mW\vy)
+ \tfrac{\mu}{2}\|\mW(\vx - \vy)\|^2.
\end{align}
By the chain rule, $\nabla g(\vy) = \mW^\top \nabla f(\mW\vy)$, and using 
$\|\mW(\vx - \vy)\| \ge \sigma_{\min}(\mW)\|\vx - \vy\| 
\ge \tfrac{1}{1+\epsilon}\|\vx - \vy\|$,
we have
\begin{align}
g(\vx) \ge g(\vy) 
+ \nabla g(\vy)^\top(\vx - \vy)
+ \tfrac{\mu}{2}\!\left(\tfrac{1}{1+\epsilon}\right)^{\!2}\!\|\vx - \vy\|^2.
\end{align}
Thus $g$ is $(r', \mu')$-locally strongly convex with $\mu' = \mu(1+\epsilon)^{-2}$.

Finally, for any $\vx$ satisfying $\|\vx - \vx^*_g\| \le r' = r / \sigma_{\max}(\mW)$,  
we have 
$\|\mW\vx - \vx_f^*\| 
\le \|\mW\| \cdot \|\vx - \vx_g^*\|
\le \sigma_{\max}(\mW) r' = r$,  
which implies $\mW\vx \in \sD_{\vx_f^*}$.
Hence $\vx \in \sD_{\vx^*_g}$, completing the proof.
\end{proof}

\subsection{Proof of Lemma~\ref{lemma:ball_in_new_bowl}}
\label{appendix:ball_in_new_bowl}

\begin{proof}
For any $\vx \in \sR^d$ satisfying $\|\vx - \vx^*_g\| \le r(1-\epsilon)$, we have
\begin{align}
\|\mW\vx - \vx^*_f\|
= \|\mW(\vx - \vx^*_g)\|
\le \|\mW\| \, \|\vx - \vx^*_g\|
\le \tfrac{1}{1-\epsilon} \, r(1-\epsilon)
= r.
\end{align}
Hence $\mW\vx \in \sD_{\vx^*_f}$, which implies $\vx \in \sD_{\vx^*_g}$. 
Therefore, the ball of radius $r(1-\epsilon)$ around $\vx^*_g$ is fully contained within the preimage of the strongly convex basin $\sD_{\vx^*_f}$.
\end{proof}

\subsection{Proof of 
Lemma~\ref{lemma:converge_to_shifted_minimizer}}
\label{appendix:converge_to_shifted_minimizer}

\begin{proof}
\textcolor{black}{
Since $g(\vx) = f(\mW\vx)$, we have 
$\nabla g(\vx) = \mW^\top \nabla f(\mW\vx)$. Setting 
$\nabla g(\vx_g^*) = 0$ and using full rank of $\mW$ 
gives $\nabla f(\mW\vx_g^*) = 0$, hence 
$\mW\vx_g^* = \vx_f^*$ and 
$\vx_g^* = \mW^{-1}\vx_f^*$.}

\textcolor{black}{The near-identity assumption 
$\|\mI - \mW\| \le \frac{\epsilon}{1+\epsilon} < 1$ 
ensures that the Neumann series 
$\mW^{-1} = (\mI - (\mI - \mW))^{-1} 
= \sum_{k=0}^{\infty}(\mI - \mW)^k$ converges, giving
\begin{align}
\|\mW^{-1} - \mI\| 
= \Bigl\|\sum_{k=1}^{\infty}(\mI - \mW)^k\Bigr\| 
\le \frac{\|\mI - \mW\|}{1 - \|\mI - \mW\|} 
\le \frac{\epsilon/(1+\epsilon)}{1/(1+\epsilon)} 
= \epsilon.
\end{align}}
The shift between the two minimizers is therefore bounded by
\begin{align}
\|\vx_f^* - \vx_g^*\|
= \|(\mI - \mW^{-1})\vx_f^*\|
\le \|\mI - \mW^{-1}\|\,\|\vx_f^*\|
\le \epsilon \|\vx_f^*\|.
\end{align}
\textcolor{black}{Substituting 
$\epsilon = \frac{r-c}{r-c+\|\vx_f^*\|}$ we have:
\begin{align}
\epsilon\|\vx_f^*\| 
= \frac{(r-c)\|\vx_f^*\|}{r-c+\|\vx_f^*\|} 
= (r-c)\!\left(1 - \frac{\|\vx_f^*\|}
{r-c+\|\vx_f^*\|}\right) 
= (r - c)(1-\epsilon).
\end{align}}

If $\|\vx - \vx_f^*\| \le c(1-\epsilon)$, then by the 
triangle inequality,
\begin{align}
\|\vx - \vx_g^*\| 
\le \|\vx - \vx_f^*\| + \|\vx_f^* - \vx_g^*\|
\le c(1-\epsilon) + (r - c)(1-\epsilon)
= r(1-\epsilon),
\end{align}
which implies $\vx \in \sD_{\vx_g^*}$ by 
Lemma~\ref{lemma:ball_in_new_bowl}. 

Since $g$ is \textcolor{black}{$\beta'$-smooth and 
$(r',\mu')$-locally strongly convex} in $\sD_{\vx_g^*}$ 
\textcolor{black}{by Lemma~\ref{lemma:linear_smooth_convex} 
(with $\beta' = \beta(1+\epsilon)^2$, 
$\mu' = \mu(1+\epsilon)^{-2}$, $r' = r(1-\epsilon)$)}, 
Lemma~\ref{lemma:gd_converge} ensures that 
$\vx_k \to \vx_g^*$ for a step size 
\textcolor{black}{$\eta \le \min\ \bigl(\frac{1}{\beta'},\; 
\frac{r'-\|\vx_k - \vx_g^*\|}
{\beta' \|\vx_k - \vx_g^*\|}\bigr)$}.

Note that the effective smoothness $\beta$ may vary slightly as the environment evolves.
\end{proof}

\subsection{Proof of Theorem~\ref{theorem}}
\label{appendix:theorem}

\begin{proof}
We begin with $\vx_0^*$, a benign minimizer of $f_0$, and initialize GD within its locally strongly convex and $\beta$-smooth basin $\sD_{\vx_0^*}$.  
When the loss transitions from $f_0$ to $f_1(\vx) = f_0(\mW_1\vx)$, the spectral constraint on $\mW_1$ ensures the minimizer shift is bounded:
\begin{align}
\|\vx_0^* - \vx_1^*\|
= \|\vx_0^* - \mW_1^{-1}\vx_0^*\|
\le \|\mI - \mW_1^{-1}\|\,\|\vx_0^*\|
\le \epsilon \|\vx_0^*\| 
\le (r - c)(1-\epsilon).
\end{align}
During optimization, Lemma~\ref{lemma:gd_converge} guarantees that the GD iterates $\vx_k$ contract toward $\vx_0^*$, eventually satisfying $\|\vx_k - \vx_0^*\| \le c(1-\epsilon)$.  
By Lemma~\ref{lemma:converge_to_shifted_minimizer}, this implies $\vx_k \in \sD_{\vx_1^*}$, ensuring convergence to the new minimizer $\vx_1^*$.

Assuming the result holds for $\vx_i^*$, the same argument applies recursively to the transition $f_i \mapsto f_{i+1}$ under the corresponding $\mW_{i+1}$ satisfying the same spectral condition.  
By induction, GD converges to each successive shifted minimizer $\vx_1^*, \vx_2^*, \dots, \vx_t^*$.
\end{proof}

\section{Experiment Details}
\label{appendix:B}

\subsection{Tasks and Models}

\subsubsection{Trainability Evaluations}
\paragraph{{Random Image Labeling.}} To simulate continual learning, the Random Image Labeling environment randomly relabels the images as a new task. It has been studied the most in loss of plasticity research~\cite{lyle_understand,lyle_disentangle,lewandowski2024directionscurvatureexplanationloss,kumar2024maintaining, tang2025churn} and is also adopted in our evaluation. In particular, we consider three image benchmark datasets with increasing scales: MNIST~\citep{mnist}, CIFAR-10~\citep{he2016deep}, and Tiny-ImageNet~\citep{le2015tiny}. Most previous plasticity studies have used selected datasets for their evaluation of trainability, but we include all three in our work.

For MNIST, we utilize a 4-layer Multi-Layer Perceptron (MLP) with 256 units per layer. Each hidden layer consists of a linear transformation followed by Layer Normalization and a ReLU nonlinearity. For the more complex CIFAR-10 and Tiny-ImageNet tasks, we employ an off-the-shelf ResNet-18 \citep{he2016deep} architecture comprising four cascaded residual blocks with batch normalization and skip connections, totaling approximately 11 million parameters.

\paragraph{Random Seq2Seq.} 
To enrich the evaluation of trainability preservation, we propose a Random Seq2Seq environment, following the spirit of Random Image Labeling. The  Random Seq2Seq environment presents a similar challenge of memorizing arbitrary mappings but in the text domain.

We adopt a pre-trained T5-small model \citep{raffel2020exploring} but only keep the first Transformer block. Note that Transformers are powerful models; however, our Random Seq2Seq is synthetic and relatively simple. Therefore, we restrict the Transformer capacity to study the phenomenon of loss of plasticity in our research environment.

\subsubsection{Generalizability Evaluations}

\paragraph{Random Pixel Permuting:} This environment random permutes the input image pixels for each task \citep{goodfellow2013empirical}. It is a standard benchmark for studying generalization in supervised continual learning \citep{Dohare2024,kumar2024maintaining,lewandowski2025learning}, a protocol we similarly adopt for our evaluation. We evaluate this setting using the EMNIST dataset~\citep{cohen2017emnist}, employing the same 4-layer MLP model utilized for the Random MNIST Labeling task.

\paragraph{Bigram Cipher:} To diversify the evaluation of generalization, we propose a Random Bigram Cipher task. This is a synthetic sequence-to-sequence task where the output sequence is generated by applying a cipher to each bigram of the input sequence. We evaluate this task using the same T5-small architecture employed in the Random Seq2Seq task.

\subsection{General Setups}

All experiments were optimized using Adam~\citep{adam}. 
The learning rate was selected via a sweep over 
$\{0.005, 0.001, 0.0005, 0.0001\}$; a learning rate of $0.001$ performed best for MNIST, Tiny-ImageNet, and the Bigram Cipher task, while $0.0001$ was optimal for the remaining experiments. All results are averaged over $5$ random seeds, and shaded regions in figures indicate the standard error of the mean.

\textcolor{black}{For our interpolation and task sampling methods, the key hyperparameter is the interpolation/sampling granularity, i.e., the number of intermediate micro-tasks per task transition. A finer granularity preserves more plasticity (see \S\ref{subsec:analysis}). We used $50$ intermediate steps per transition across all experiments for consistency.}

All experiments were conducted on a computer with four NVIDIA Quadro RTX 6000 GPUs (24GB each) and an Intel Core i9-9940X CPU (14 cores).

\subsection{Baseline and Hyperparameters}

\textcolor{black}{For all baseline methods, hyperparameters were 
tuned via grid search over the following candidate values:
\begin{itemize}
    \item \textbf{L2/Spectral Regularization} Regularization strength is selected from $\{0.1, \allowbreak\ 0.01, \allowbreak\ 10^{-3}, \allowbreak\ 10^{-4}, \allowbreak\ 10^{-5}, 10^{-6}, \allowbreak\ 10^{-7}\}$.
    \item \textbf{Shrink\&Perturb.} The shrink factor is chosen from \mbox{$\{0.5, 0.7, 0.9\}$}, and the perturb factor from \mbox{$\{0.1, 0.01, 10^{-3}, 10^{-4}, 10^{-5}\}$}.
    \item \textbf{ReDO.} Recycling frequency is selected from \mbox{$\{10, 100, 1000, 10000\}$} iterations.
\end{itemize}
The best configuration for each baseline was selected 
independently per benchmark.}

\subsection{Dataset and Training Configurations}

\textbf{MNIST.} $28 \times 28$ grayscale images from 10 
classes. We uniformly sampled $5120$ training examples 
across classes and used a batch size of $512$. All methods 
were trained for $240$ epochs per task.

\textbf{CIFAR-10.} $32 \times 32$ color images from 10 
classes. All $50{,}000$ training images were used with a 
batch size of $250$ for $120$ epochs per task.

\textbf{Tiny-ImageNet.} $64 \times 64$ color images 
spanning 200 classes. All $100{,}000$ training samples 
were used with a batch size of $250$ for $120$ epochs 
per task.

\textbf{Random Seq2Seq.} Synthetic text generated from 
the English alphabet with word length $5$, input sequence 
length $50$, target length $5$, and $5120$ total samples. 
Batch size of $64$, trained for $900$ epochs per task.

\textbf{Bigram Cipher.} Synthetic text generated from 
the English alphabet with word length $5$, input sequence 
length $10$, target length $10$. $1280$ samples for 
training and $512$ for testing. Batch size of $64$, 
trained for $200$ epochs per task.

\section{\textcolor{black}{Analysis on Task Transition Granularity}}
\label{appendix:C}

We investigate how interpolation granularity affects the loss landscape and internal network behavior by tracking three properties across different granularity levels. These properties are evaluated on the Random Image Labeling benchmark using a 4-layer MLP and the Bigram Cipher benchmark using T5-small.

\paragraph{Maximum Hessian Eigenvalue.}
The maximum eigenvalue of the Hessian measures the sharpness of the loss landscape~\citep{dinh_sharpness,lewandowski2024directionscurvatureexplanationloss,lyle2022understanding}.
A higher value indicates a sharper region, which is associated with greater difficulty in optimization. We compute the maximum Hessian eigenvalue at checkpoints after each task transition, under varying numbers of interpolation steps. As shown in Figure~\ref{fig:granularity_eigen}, finer-grained transitions consistently produce lower maximum eigenvalues, indicating that the optimizer stays in a flatter, better-conditioned region. This aligns with our theoretical analysis that local smoothness is preserved under small task changes.

\paragraph{Relative Representation Change.}
Relative representation change captures how much hidden
representations shift between consecutive tasks, serving as a
proxy for the network's ability to
learn~\citep{lewandowski2025learning}. We measure the change in
hidden representations between consecutive task checkpoints on a
fixed probe batch:
$$\mathrm{RelRepChange} = \mathbb{E}_{x}\left[
\frac{\|h_{\mathrm{curr}}(x)-h_{\mathrm{prev}}(x)\|}
{\|h_{\mathrm{prev}}(x)\|} \right],$$
where $h(x)$ is the hidden representation of input $x$.
As shown in Figure~\ref{fig:granularity_rep}, smaller
interpolation steps better preserve the ability to update hidden
representations between consecutive tasks. This suggests that
smoother task transitions maintain the network's capacity to adapt
to new data.

\begin{figure*} [t!]
    \centering
    \includegraphics[width=0.8\linewidth]{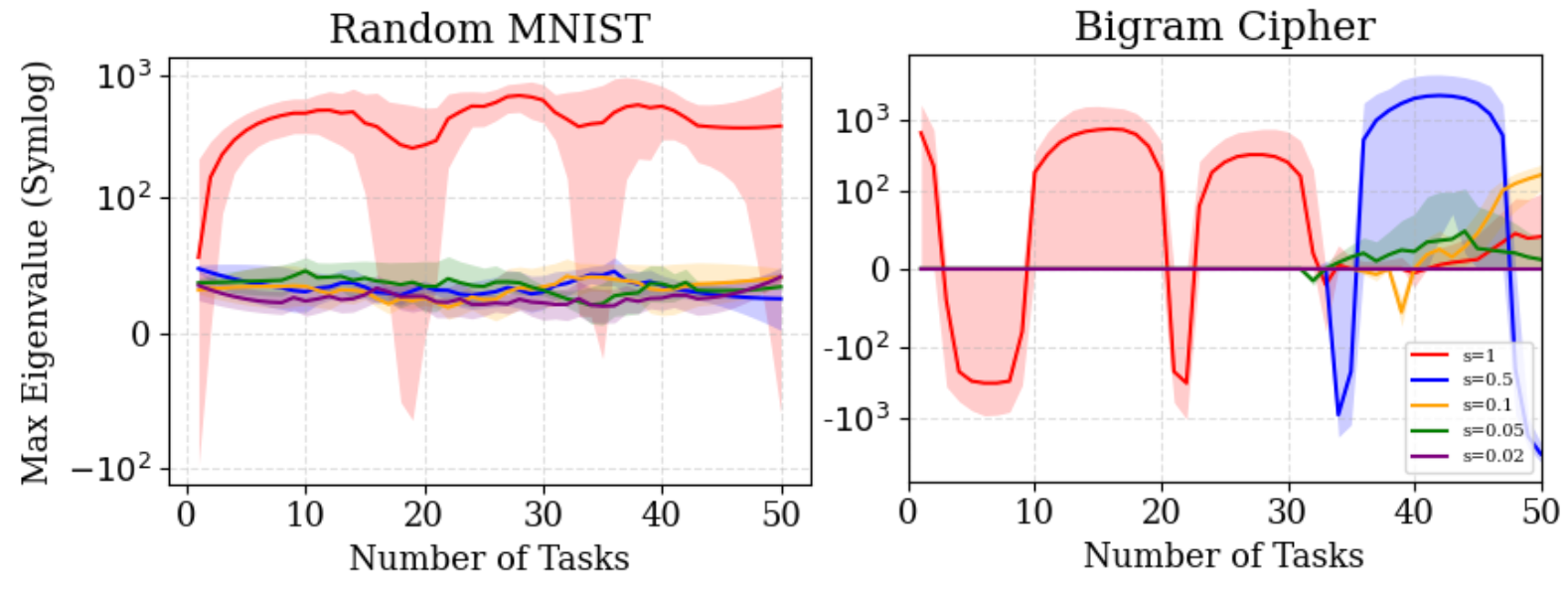}
    \caption{Finer-grained transitions consistently produce lower maximum eigenvalues, indicating that the optimizer remains in a flatter, better-conditioned region of the loss landscape. This aligns with our theoretical analysis that under small task changes, the local smoothness is preserved.}
    \label{fig:granularity_eigen}
\end{figure*}

\begin{figure*}[t!]
    \centering
    \includegraphics[width=0.8\linewidth]{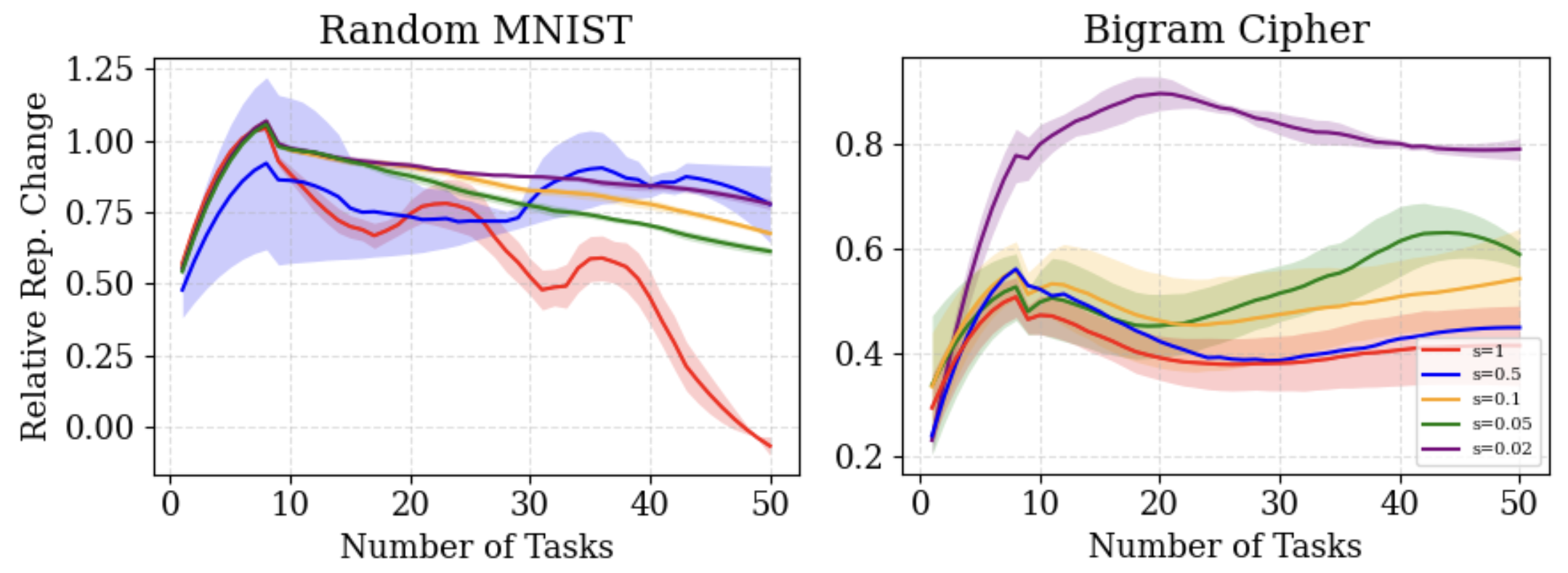}
    \caption{Smaller interpolation steps better preserve the ability to update hidden representations between consecutive tasks, indicating that smoother task changes preserve the network’s ability to learn (plasticity).}
    \label{fig:granularity_rep}
\end{figure*}

\begin{figure*}[t!]
    \centering
    \includegraphics[width=0.8\linewidth]{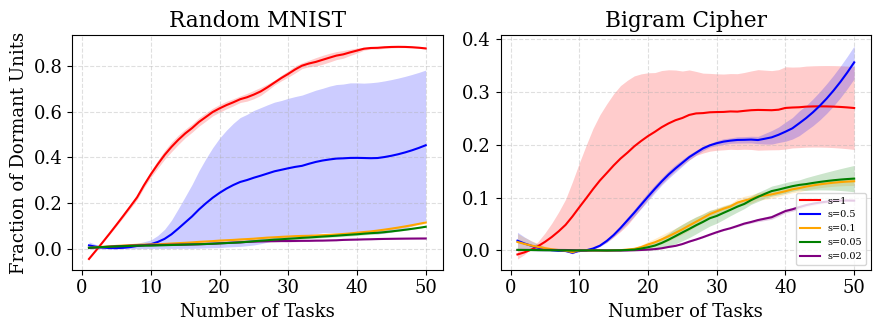}
    \caption{Finer-grained transitions consistently have a lower fraction of dormant neurons while also preserving plasticity.}
    \label{fig:granularity_dormant}
\end{figure*}

\paragraph{Dormant Unit Fraction.}
A dormant neuron is one that rarely fires, and a high dormant
fraction has been linked to plasticity
loss~\citep{Dohare2024, dormantSokar}. Following~\citet{Dohare2024},
a unit is dormant if its post-ReLU activation is positive on fewer
than $\tau = 1\%$ of probe inputs:
$$\mathrm{DormantFraction} = \frac{ \sum_{l=1}^{L}\sum_{i=1}^{H^l}
\mathbf{1}\!\left[ \frac{1}{N}\sum_{n=1}^{N}\mathbf{1}\!\left[
a_{i,n}^l > 0\right] \le \tau \right] }{\sum_{l=1}^{L} H^l},$$
where $l$ is the layer index, $H^l$ the number of hidden units
in layer $l$, $N$ the number of probe inputs, and $a_{i,n}^l$
the post-ReLU activation of neuron $i$ in layer $l$ on input $n$.
As shown in Figure~\ref{fig:granularity_dormant}, finer-grained
transitions consistently produce fewer dormant neurons. This is
consistent with prior findings that plasticity loss is closely
tied to neuron dormancy, and suggests that gradual transitions
help keep more of the network active and trainable.

\section{Limitations and Future Work}
\label{appendix:D}
Our theoretical analysis assumes Lipschitz smoothness and local strong convexity of the loss landscape, which may not hold for all real-world tasks. That said, both are standard assumptions in optimization theory. Local strong convexity naturally arises near well-behaved minima through second-order Taylor approximation, and Lipschitz smoothness is widely adopted in convergence analysis~\citep{bian2024make,wu2024meta}.

For empirical analysis, we evaluated on four tasks, which may not cover all possible environments. However, we selected image benchmarks widely used in prior work and additionally proposed more challenging language tasks to broaden the evaluation. Therefore, we believe our experiments are broad and convincing.

In the future, we are interested in automatically pre-determining the appropriate task transition granularity for a given continual learning application, rather than relying on manual selection.

\end{document}